\documentclass[conference]{IEEEtran}
\usepackage{times}
\usepackage{cite}
\usepackage[T1]{fontenc} % Use Type 1 fonts
\usepackage{textcomp}   % Additional symbol support
\usepackage{amsmath, amssymb}
\usepackage{graphicx}
\usepackage{subcaption}
\usepackage{color}
\usepackage{mathtools}
\usepackage{algorithm}
\usepackage{algorithmicx}
\usepackage{algpseudocode}
\usepackage{booktabs} 
\usepackage{caption}
\usepackage{url}
\usepackage{pifont}
\usepackage{xcolor}
\usepackage{footmisc}
\usepackage{caption}

\usepackage[numbers]{natbib}
\usepackage{multicol}
\usepackage[bookmarks=true]{hyperref}

\newcommand{\cN}{\mathcal{N}}

\newcommand{\mR}{\mathbb{R}}
\newcommand{\mE}{\mathbb{E}}

\newcommand{\mY}{{\mathbb Y}}

\newcommand{\dsdf}{d_{\mathrm{sdf}}}  % SDF query function
\newcommand{\Kelite}{K_{\mathrm{elite}}}  % Elite set size
\newcommand{\Kmax}{K_{\mathrm{max}}}  % Maximum iterations

\usepackage{amsthm}  
\newtheorem{theorem}{Theorem}

\newtheorem{lemma}{Lemma}

\newtheorem{remark}{Remark}

\title{PISTO: Proximal Inference for Stochastic Trajectory Optimization}

\IEEEoverridecommandlockouts
\author{ Hongzhe Yu, Zinuo Chang, and Yongxin Chen
\thanks{H. Yu and Y.\ Chen are with the School of Aerospace Engineering, Z. Chang is with the School of Electrical and Computer Engineering, Georgia Institute of Technology, Atlanta, GA; {\{hyu419, zchang40, yongchen\}@gatech.edu}}
}

\begin{document}

\maketitle

\begin{abstract}
Stochastic trajectory optimization methods like STOMP enable planning with non-differentiable costs, offering substantial flexibility over gradient-based approaches. We show that STOMP implicitly minimizes the KL divergence from a Boltzmann trajectory distribution, revealing an elegant Variational Inference (VI) structure underlying its updates. Building on this insight, we propose the \textit{Proximal Inference for Stochastic Trajectory Optimization} (PISTO) algorithm that stabilizes the updates by augmenting the objective with a KL regularization between successive Gaussian proposals. This proximal formulation admits a trust-region interpretation and yields closed-form mean updates computable as expectations under a surrogate distribution. We estimate these expectations via importance-weighted Monte Carlo sampling, producing a simple, derivative-free algorithm that inherits STOMP's ability to handle non-differentiable and discontinuous costs without modification. On robot arm motion planning benchmarks, PISTO achieves an 89\% success rate---outperforming CHOMP (63\%) and STOMP (68\%)---while producing shorter, smoother paths at twice the speed of competing stochastic methods. We further validate PISTO on contact-rich MuJoCo locomotion and manipulation tasks, where it consistently outperforms both CEM and MPPI baselines in reward.
\end{abstract}

\IEEEpeerreviewmaketitle

\section{Introduction}
Motion planning is central to robotics, enabling autonomous systems to generate collision-free, dynamically feasible trajectories in cluttered environments. While sampling-based methods such as PRM and RRT$^\star$~\cite{karaman2011sampling, phillips2012asymptotically} provide strong exploration and asymptotic guarantees, their trajectories often require substantial post-processing for smoothness and dynamic feasibility. Consequently, trajectory optimization has become essential for high-performance planning in manipulation, mobile robotics, and aerial systems.

Optimization-based planners—including CHOMP~\cite{ratliff2009chomp, zucker2013chomp}, sequential convex programming~\cite{schulman2014motion}, and Gaussian process methods~\cite{mukadam2017gaussian, mukadam2018gaussian}—represent trajectories in continuous time and optimize for smoothness, collision avoidance, and dynamic constraints. These approaches achieve impressive speed and scalability, but they rely on differentiable cost representations and remain susceptible to local minima in highly nonconvex environments.

Stochastic trajectory optimization broadens the class of tractable objectives by introducing sampling. STOMP~\cite{kalakrishnan2011stomp} and path-integral methods such as PI$^2$~\cite{theodorou2010generalized} perturb candidate trajectories and update them via cost-weighted averaging, enabling optimization under non-differentiable and discontinuous costs. A natural question arises: what objective does STOMP implicitly optimize?

We show that STOMP minimizes the reverse KL divergence from a Boltzmann trajectory distribution, revealing an elegant variational inference structure underlying its updates. This connection places STOMP within a broader inference-based planning framework~\cite{kappen2005path, todorov2009efficient, chen2016relation}, where trajectory costs define a Boltzmann distribution whose high-probability mass concentrates on desirable plans. Gaussian variational inference formulations (GVIMP)~\cite{yu2023gvimp} approximate this distribution with a structured Gaussian family, providing uncertainty quantification and a principled KL-divergence objective.

Building on this insight, we develop a proximal inference algorithm that stabilizes STOMP's updates. Proximal methods~\cite{parikh2014proximal, chen2022improved} augment each iterate with a penalty on the distance from the current solution, effectively creating a trust region that prevents overly aggressive steps. In the distributional setting, this naturally extends to KL-divergence penalties between successive proposals, controlling how rapidly the trajectory distribution evolves. By combining the variational inference perspective with proximal regularization, we obtain closed-form mean updates computable via importance-weighted Monte Carlo sampling—a simple, derivative-free algorithm that inherits STOMP's flexibility while offering principled stability guarantees. 

\textbf{\textit{Contributions:}} 
\begin{itemize}
    \item We reveal that STOMP implicitly minimizes KL divergence from a Boltzmann trajectory distribution, establishing its variational inference structure and motivating principled algorithmic improvements.
    
    \item We propose \emph{Proximal Inference for Stochastic Trajectory Optimization (PISTO)}, which augments the reverse-KL objective with a proximal penalty between successive Gaussian proposals. This formulation admits a trust-region interpretation, stabilizes optimization dynamics, and yields closed-form mean updates computable via importance-weighted Monte Carlo sampling.
    
    \item We demonstrate PISTO's effectiveness across diverse robotics tasks: on motion planning benchmarks, PISTO achieves 89\% success rate—outperforming CHOMP (63\%), STOMP (68\%), and natural gradient descent (76\%)—while running twice as fast as competing stochastic methods. On contact-rich MuJoCo tasks, PISTO consistently outperforms CEM and MPPI baselines in reward despite non-differentiable dynamics.
\end{itemize}

\section{Related Work}
\label{sec:related_work}

\paragraph{Planning as Inference}
The interpretation of optimal control as probabilistic inference connects to linearly-solvable control~\cite{todorov2009efficient}, path-integral methods~\cite{kappen2005path}, and variational formulations in RL~\cite{levine2018reinforcement}. Within motion planning, Toussaint~\cite{toussaint2009robot} applied message passing on graphical models, while Stein Variational approaches~\cite{power2024constrained} maintain trajectory distributions via particle-based inference. Our work builds on the Planning as Inference paradigm~\cite{yu2023gvimp, cosier2024unifying, petrovic2019stochastic, petrovic2022mixtures, chang2025efficient} , which showed that trajectory optimization with control-energy regularization is equivalent to KL-divergence minimization from a Boltzmann distribution. We solve the inference problem using a proximal inference algorithm that provides stable optimization and efficient Monte Carlo estimation. Other works under this paradigm are closely related. The variational and heteroscedasticity GP planners \cite{cosier2024unifying, petrovic2019stochastic} either rely on pre-defined kernels, induction point approximations, or start and goal states, which cannot be easily extended to control space optimization. Mixtures of Gaussian Processes for Trajectory Optimization (MGPTO) \cite{petrovic2022mixtures} obtains a multi-modal trajectory using a STOMP-style cost-weighted stochastic gradient estimate. P-GVIMP \cite{chang2025efficient} leverages GPU parallel computation to accelerate the mean and covariance updates in a gradient-descent landscape.

\paragraph{Trajectory Optimization}
Gradient-based methods such as CHOMP~\cite{ratliff2009chomp, zucker2013chomp}, TrajOpt~\cite{schulman2014motion}, and GPMP~\cite{mukadam2017gaussian} achieve fast convergence but require differentiable costs. Sampling-based methods relax this requirement: STOMP~\cite{kalakrishnan2011stomp} uses cost-weighted averaging, path integral methods (PI$^2$~\cite{theodorou2010generalized}, MPPI~\cite{williams2017mppi}) derive updates from stochastic control, and CEM~\cite{rubinstein1999cem} fits distributions to elite samples. Our method differs by: (i) revealing STOMP's variational inference structure via KL-divergence; (ii) introducing proximal regularization for trust-region stability; and (iii) deriving closed-form updates amenable to importance sampling.

Our method introduces three key innovations over these approaches: (i) an explicit connection between STOMP's objective and KL-divergence minimization, revealing its variational inference structure; (ii) proximal regularization that stabilizes updates with a trust-region interpretation; and (iii) a reversed KL formulation yielding closed-form mean updates amenable to unbiased importance sampling. 

\section{Stochastic Trajectory Optimization as Variational Inference Planning}
\label{sec:framework}

In this section, we establish the theoretical foundation of our approach by connecting classical stochastic trajectory optimization with variational inference. We begin by introducing the STOMP framework, followed by its interpretation as a variational inference problem.

\subsection{Stochastic Trajectory Optimization (STOMP)}
Stochastic Trajectory Optimization for Motion Planning (STOMP) \cite{kalakrishnan2011stomp} is a gradient-free optimization framework designed to handle non-differentiable and discontinuous cost functions. It explores the trajectory space by generating stochastic perturbations around a nominal trajectory and updating it based on the exponentiated cost of these samples.

Let $Y \in \mathbb{R}^{n(T+1)}$ represent the mean of a $T$-length state trajectory, defined as $Y \triangleq \{x_t\}_{t=0}^T$. STOMP seeks to minimize the expectation of a cost functional over a proposal Gaussian distribution $\tilde{Y} \sim \mathcal{N}(Y, \Sigma)$:
\begin{equation}
\label{eq:obj_stomp}
\min_Y \mathcal{J}_1(Y) \triangleq \mathbb{E}_{\tilde{Y}} \left[ S(\tilde{Y}) + \frac{1}{2} \tilde{Y}^\top R \tilde{Y} \right],
\end{equation}
where $S(\tilde{Y}) = \sum_{t=0}^T V(\tilde{Y}_t)$ is the state-dependent potential (e.g., obstacle avoidance) and $R = A^\top A$ is the control cost matrix for double integrator dynamics, with $A \in \mathbb{R}^{(T-1) \times (T+1)}$ the finite-difference acceleration operator applying the stencil $[1,-2,1]/\Delta t^2$ at interior nodes.

This stochastic formulation allows the planner to "smooth" the cost landscape, effectively escaping local minima that often trap purely gradient-based methods.

\subsection{Variational Inference Interpretation of STOMP}
The objective \eqref{eq:obj_stomp} admits a Variational Inference (VI) interpretation that recasts trajectory optimization as approximating an optimal path distribution. Any cost function over trajectories induces a Boltzmann distribution
\begin{equation}
p^\star(\tilde{Y}) \propto \exp\left( -S(\tilde{Y}) \right),
\end{equation}
where low-cost trajectories receive high probability mass; this construction underlies the connection between optimal control and inference~\cite{kappen2005path, todorov2009efficient} and path-integral methods~\cite{theodorou2010generalized}. In \eqref{eq:obj_stomp}, the state-dependent potential $S(\tilde{Y})$ captures task-specific objectives such as obstacle avoidance, while the control cost $\frac{1}{2}\tilde{Y}^\top R \tilde{Y}$ with $R = A^\top A$ corresponds to a Gaussian smoothness prior $\mathcal{N}(0, R^{-1})$ favoring low-acceleration paths. Combining the two yields the main result connecting STOMP to VI.

\begin{theorem}[Variational Inference Formulation of STOMP]
\label{thm:stomp_vi}
Consider the stochastic trajectory optimization objective \eqref{eq:obj_stomp}. Let $\Sigma \succ 0$ be a fixed covariance and $R \succ 0$ be the control cost matrix. Minimizing $\mathcal{J}_1$ over the mean trajectory $Y$ is equivalent to solving the variational inference problem:
\begin{equation}
\label{eq:obj_STOMP_KL}
\min_Y \; D_{\mathrm{KL}}\left( \mathcal{N}(Y, \Sigma) \;\|\; \mY^\star \right),
\end{equation}
where the target distribution is the energy-based posterior:
\begin{equation}
\mY^\star \propto \exp\left( -S(\tilde{Y}) \right) \mathcal{N}(0, R^{-1}).
\end{equation} 
\end{theorem}

\begin{proof}
See Appendix \ref{sec:appendix_proofs}-(a).
\end{proof}

\section{Proximal Inference for Stochastic Trajectory Optimization}
\label{sec:methods}
We introduce a novel paradigm for solving the motion planning inference problem: Proximal Inference for Stochastic Trajectory Optimization (PISTO). PISTO reverses the argument order in the KL divergence and augments the objective with a Gaussian proximal term, converting each iteration into a proximal inference minimization whose optimizer admits a closed-form expectation representation amenable to Monte Carlo estimation. 

\subsection{Solution to the Reverse KL Minimization Problem}
We first introduce an important moment-matching solution result in reverse KL minimization problems.
\begin{lemma}[Reverse KL Minimization and Moment-matching Solution]
\label{lem:cem}
Consider the reverse KL objective obtained by swapping distributions in \eqref{eq:obj_STOMP_KL}:
\begin{equation}
\label{eq:obj_STOMP_reverseKL}
\min_Y \; D_{\mathrm{KL}} \left( \mY^\star \parallel \cN(Y, \Sigma) \right),
\end{equation}
where $\mY^\star \propto e^{-S(\Tilde{Y})} \cN(0, R^{-1})$. The gradient of the objective \eqref{eq:obj_STOMP_reverseKL} with respect to $Y$ is
\begin{align}
\label{eq:grad_CE}
    \nabla_Y D_{\mathrm{KL}} \left( \mY^\star \parallel \cN(Y, \Sigma) \right) 
    & = \mE_{\mY^\star} \left[ - \Sigma^{-1} (\Tilde{Y} - Y) \right],
\end{align}
and the optimal mean is given by
\begin{equation}
Y^\star = \mE_{\mY^\star}[\Tilde{Y}].
\end{equation}
With the choice $\Sigma = R^{-1}$ and proposal distribution $\Tilde{Y} \sim \cN(Y, R^{-1})$, the importance sampling estimator is
\begin{equation}
\label{eq:importance_sampling_weight}
\hat{Y}^\star = \sum_{m=1}^{M} \bar{w}_m (Y + \varepsilon_m), \quad \bar{w}_m = \frac{e^{-(S(Y + \varepsilon_m) + \varepsilon_m^\top R Y)}}{\sum_{j=1}^{M} e^{-(S(Y + \varepsilon_j) + \varepsilon_j^\top R Y)}},
\end{equation}
where $\varepsilon_m \sim \cN(0, R^{-1})$ are i.i.d.\ samples.
\end{lemma}
The proof of Lemma \ref{lem:cem} can be found in~\cite{williams2018information}.
\subsection{The main PISTO Formulation}
\begin{theorem}[Proximal Inference Update]
\label{thm:prox_vi}
Consider the proximal update for the VI problem \eqref{eq:obj_STOMP_KL}:
\begin{align}
Y_{k+1} = \arg\min_Y \; & D_{\mathrm{KL}} \left( \cN(Y,\Sigma) \parallel \mY^\star \right) 
\\
&+ \frac{1}{\eta} D_{\mathrm{KL}}\left(\cN(Y,\Sigma) \parallel \cN(Y_k,\Sigma)\right),
\end{align}
where $\eta > 0$ is the step size parameter, $Y_k$ is the current iterate, and $\mY^\star \propto e^{-S(\Tilde{Y})} \cN(0, R^{-1})$ is the target distribution. Then the proximal update is equivalent to KL projection onto a surrogate distribution:
\begin{equation}
\label{eq:prox_KL}
Y_{k+1} = \arg\min_Y \; D_{\mathrm{KL}} \left( \cN(Y,\Sigma) \parallel \mY^\star_k \right),
\end{equation}
where the surrogate target $\mY^\star_k$ admits the explicit form
\begin{equation}
\mY^\star_k \propto \left( e^{-S(\Tilde{Y})} \cN(0, R^{-1}) \right)^{\frac{\eta}{\eta+1}} \left( \cN(Y_k, \Sigma) \right)^{\frac{1}{\eta+1}}.
\end{equation}
\end{theorem}

\begin{proof}
Appendix \ref{sec:appendix_proofs}-(c).
\end{proof}

\begin{remark}
The surrogate distribution $\mY^\star_k$ geometrically interpolates between the true target $\mY^\star$ and the current Gaussian approximation $\cN(Y_k, \Sigma)$. As $\eta \to \infty$, the surrogate converges to the true target, recovering standard VI. For finite $\eta$, the proximal term regularizes the update, improving numerical stability and convergence.
\end{remark}

\subsection{Moment-matching Solution to the Proximal Inference}

Eq.~\eqref{eq:prox_KL} indicates that each proximal iteration is itself a VI problem with respect to the surrogate distribution \(\mY^\star_k\), which is typically easier to handle than the original Boltzmann distribution. To compute the update, we solve the reverse KL-minimization projection
\begin{equation}
\label{eq:CE_prob_step_k}
Y_{k+1}=\arg\min_Y D_{\mathrm{KL}} \left( \mY^\star_k  \parallel \cN(Y,\Sigma)\right).
\end{equation}
To sample from \(\mY^\star_k\), we rewrite it in a tilted-Gaussian form by completing the square. Letting $\gamma \triangleq \frac{\eta}{\eta+1}$, we obtain
\begin{align}
    \log \mY^\star_k 
    &\propto 
    -\gamma S(\Tilde{Y}) - \frac{1}{2} (\Tilde{Y}-\mu_k)^\top P(\Tilde{Y}-\mu_k),
\end{align}
where $P = \gamma R + (1-\gamma)\Sigma^{-1}$ and $\mu_k = (1-\gamma)P^{-1}\Sigma^{-1}Y_k$. This representation reveals that \(\mY^\star_k\) behaves like a Gaussian whose mean is shifted by the proximal term, modulated by an exponential tilt involving the cost function:
\begin{align}
    \mY^\star_k \propto e^{-\gamma S(\Tilde{Y})}\cN(\Tilde{Y}; \mu_k, P^{-1}).
\end{align}
Applying the gradient identity from \eqref{eq:grad_CE} to the reverse KL-minimizing objective \eqref{eq:CE_prob_step_k} and setting it to zero yields the following result.

\begin{theorem}[Moment-Matching Update]
\label{thm:moment_matching}
The solution to the reverse KL-minimization projection \eqref{eq:CE_prob_step_k} satisfies the condition
\begin{align}
    \nabla_{Y}D_{\mathrm{KL}} \left( \mY^\star_k  \parallel \cN(Y,\Sigma)\right) = \mathbb{E}_{\mY^\star_k}\left[ -\Sigma^{-1}(\Tilde{Y} - Y) \right] = 0,
\end{align}
which admits the closed-form moment-matching update
\begin{align}
\label{eq:moment_matching_update}
    Y_{k+1} = \mathbb{E}_{\mY^\star_k}[ \Tilde{Y} ].
\end{align}
That is, the next iterate is simply the mean of the surrogate distribution \(\mY^\star_k\).
\end{theorem}

\subsection{Importance Sampling}
To estimate the expectation in \eqref{eq:moment_matching_update},  we employ importance sampling from a Gaussian proposal distribution. Similar to Eq.~\eqref{eq:importance_sampling_weight}, the weight at iteration \(k\) is
\begin{align*}
    w_k(\Tilde{Y}) &\triangleq \frac{e^{-\gamma S(\Tilde{Y})} \mathcal{N}(\Tilde{Y}; \mu_k, P^{-1})}{\mathcal{N}(\Tilde{Y}; Y_k, P^{-1})}  \propto e^{-\gamma S(\Tilde{Y}) - \Tilde{Y}^T P (Y_k - \mu_k)}. 
\end{align*}
In the common choice \(\Sigma = R^{-1}\), the expression simplifies considerably, as $P = R$ and $\mu_k = (1-\gamma)Y_k$:
\begin{align}
    w_k(\Tilde{Y}) \propto  e^{-\gamma (S(\Tilde{Y}) + \Tilde{Y}^T R Y_k) }.
\end{align}
Using these weights, we obtain a Monte-Carlo estimation of the proximal update:
\begin{align}
    Y_{k+1} \approx \sum_{m=1}^{M}\; \frac{e^{-\gamma (S(Y_k + \varepsilon_m) + \varepsilon_m^T R Y_k)}}{\sum_{j=1}^{M} e^{-\gamma (S(Y_k + \varepsilon_j) + \varepsilon_j^T R Y_k)}} (Y_k + \varepsilon_m), 
\end{align}
where $\varepsilon_m \sim \cN(0, R^{-1}), \;m=1,\dots,M$ are sampled independently. This Monte Carlo estimator completes the proximal inference update, providing a stable and efficient mechanism for refining the Gaussian approximation at each iteration.

\begin{algorithm}[t]
\caption{PISTO}
\label{alg:pisto}
\begin{algorithmic}[1]
\Require Accumulated cost $S(\cdot)$, smoothness matrix $R$, Cholesky factor $L$ with $LL^\top = R^{-1}$, sample size $M$, proximal parameter $\eta$, current iterate $Y_k$, temperature $\tau$
\Ensure Updated iterate $Y_{k+1}$
\State Compute $\gamma \gets \dfrac{\eta}{\eta + 1}$
\State \textbf{Sample:} Draw $\{\varepsilon_m\}_{m=1}^M$ where $\varepsilon_m = L z_m$, \quad $z_m \sim \mathcal{N}(0, I)$
\State \textbf{Evaluate:} For each $m = 1, \ldots, M$, compute importance weights
\[
w_m \gets \exp\!\Big(-\frac{\gamma}{\tau} \, S(Y_k + \varepsilon_m) - (\varepsilon_m)^\top R \, Y_k \Big)
\]
\State \textbf{Normalize:} Compute normalized weights
\[
\bar{w}_m \gets \frac{w_m}{\sum_{j=1}^M w^{(j)}}, \quad \forall m
\]
\State \textbf{Update:} Compute weighted mean
\[
Y_{k+1} \gets Y_k + \sum_{m=1}^M \bar{w}_m \, \varepsilon_m
\]
\State \Return $Y_{k+1}$
\end{algorithmic}
\end{algorithm}

\begin{figure*}[h]
    \centering
    \begin{subfigure}{0.23\textwidth}
    \centering
    \includegraphics[width=0.9\textwidth]{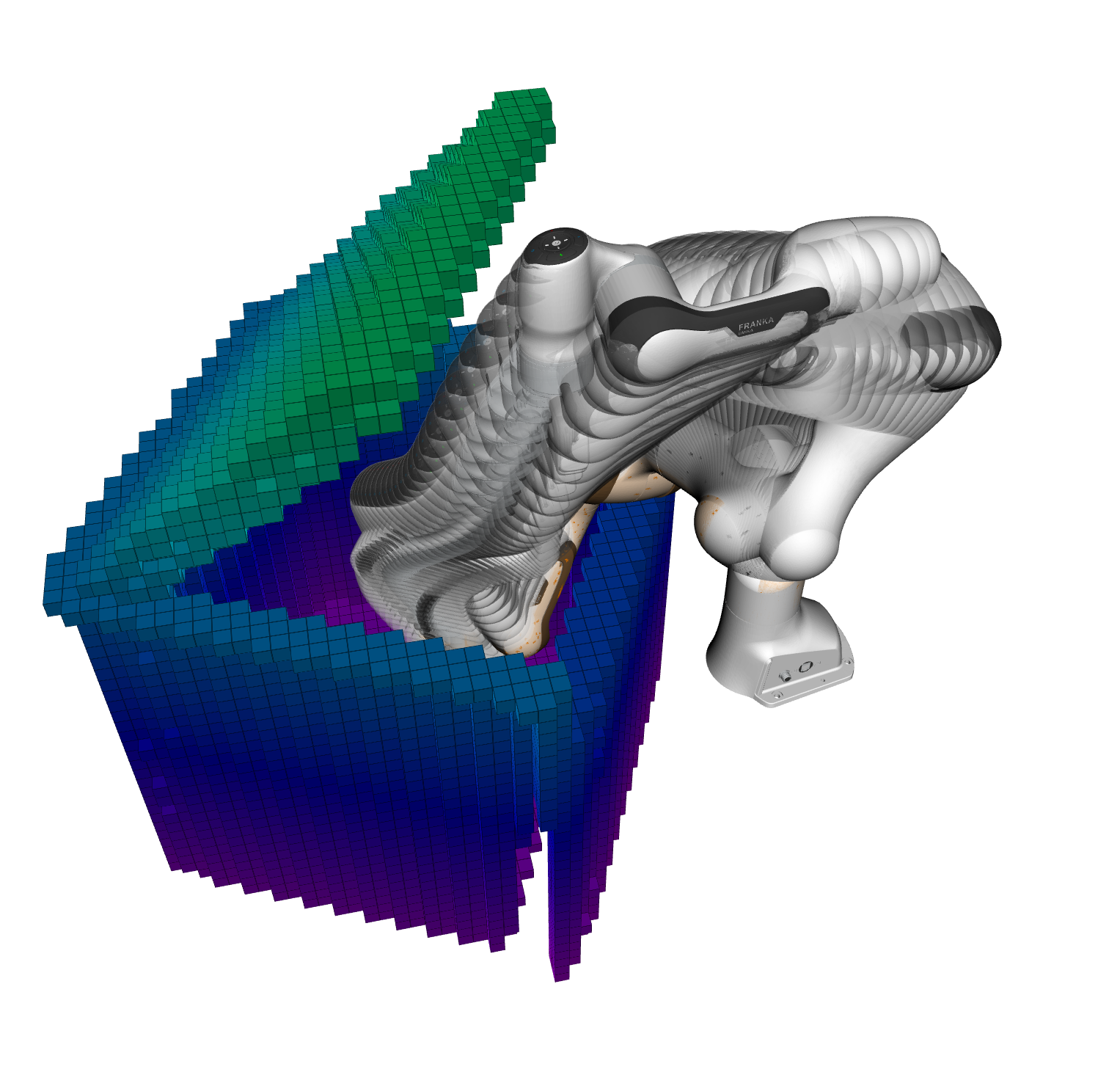}
    \caption{Box Scene}
    \end{subfigure}
    \begin{subfigure}{0.23\textwidth}
    \centering
    \includegraphics[width=0.95\textwidth]{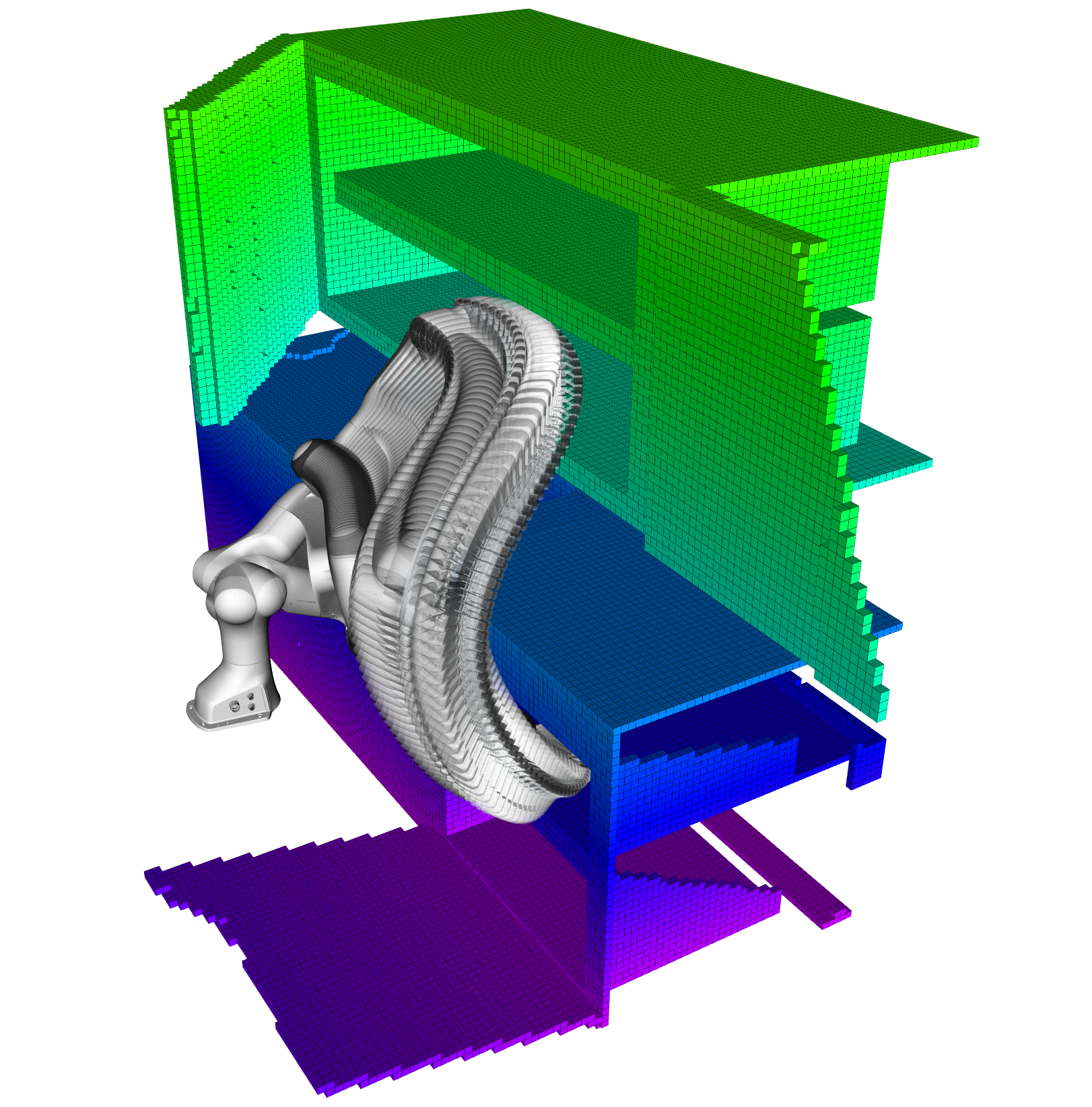}
    \caption{Kitchen Scene}
    \end{subfigure}
    \begin{subfigure}{0.23\textwidth}
    \centering
    \includegraphics[width=0.9\textwidth]{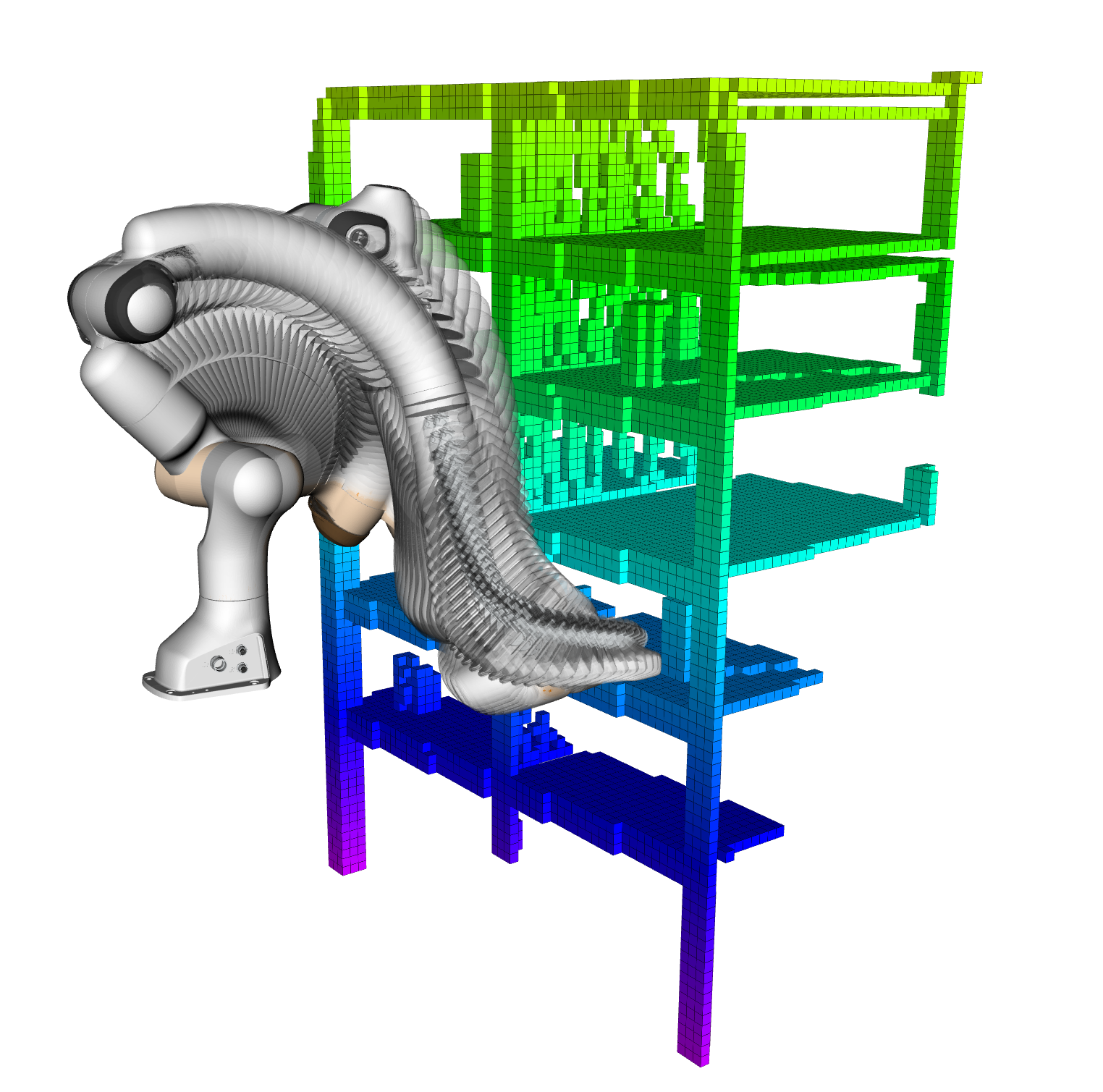}
    \caption{Bookshelf Scene}
    \end{subfigure}
    \begin{subfigure}{0.23\textwidth}
    \centering
    \includegraphics[width=0.9\textwidth]{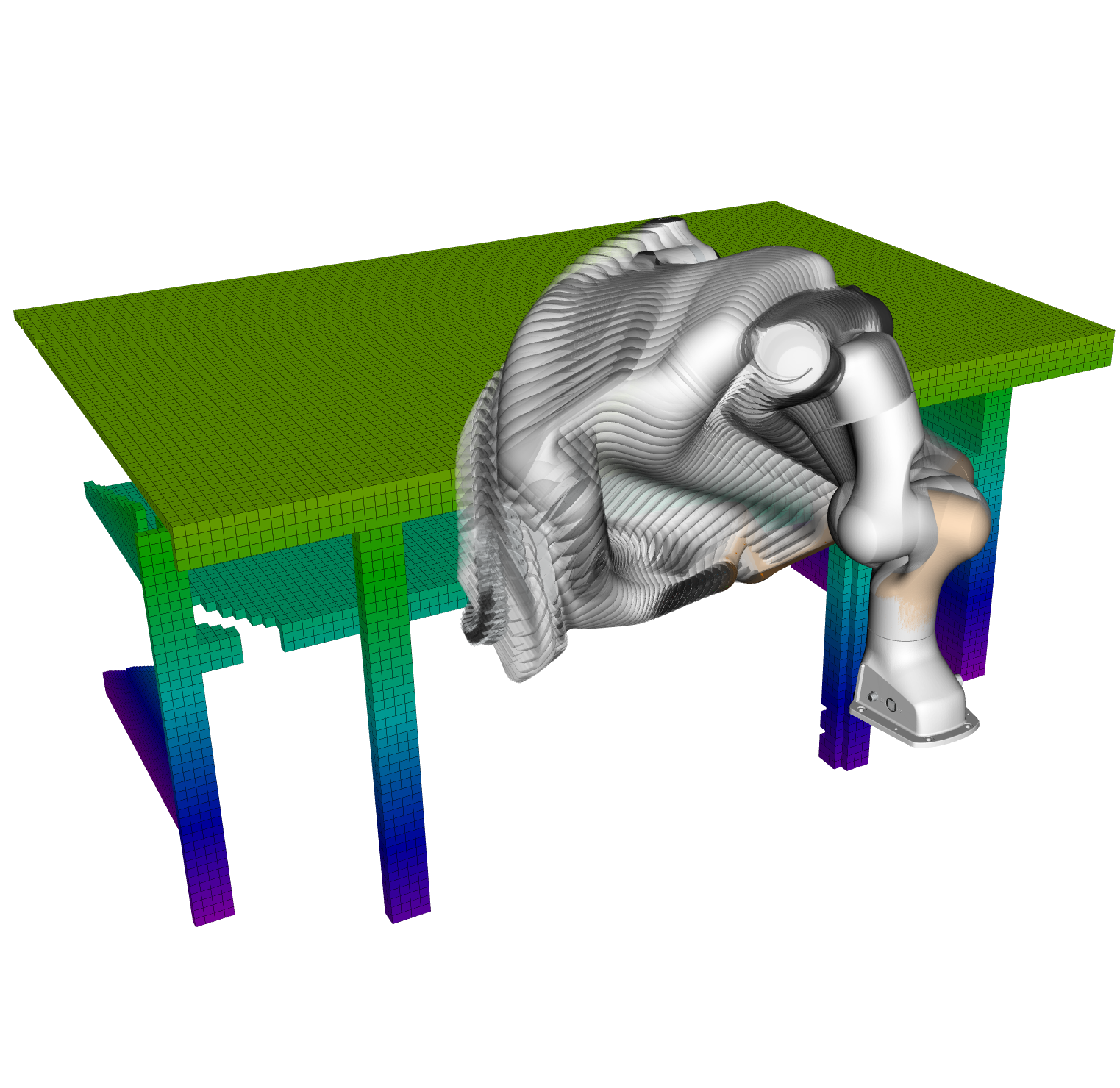}
    \caption{Table Scene}
    \end{subfigure}
    \caption{Results of PISTO in different motion planning benchmarking scenes.}
\end{figure*}

\begin{figure*}[h]
    \centering
    \begin{subfigure}{\linewidth}
    \centering
        \includegraphics[width=0.9\linewidth]{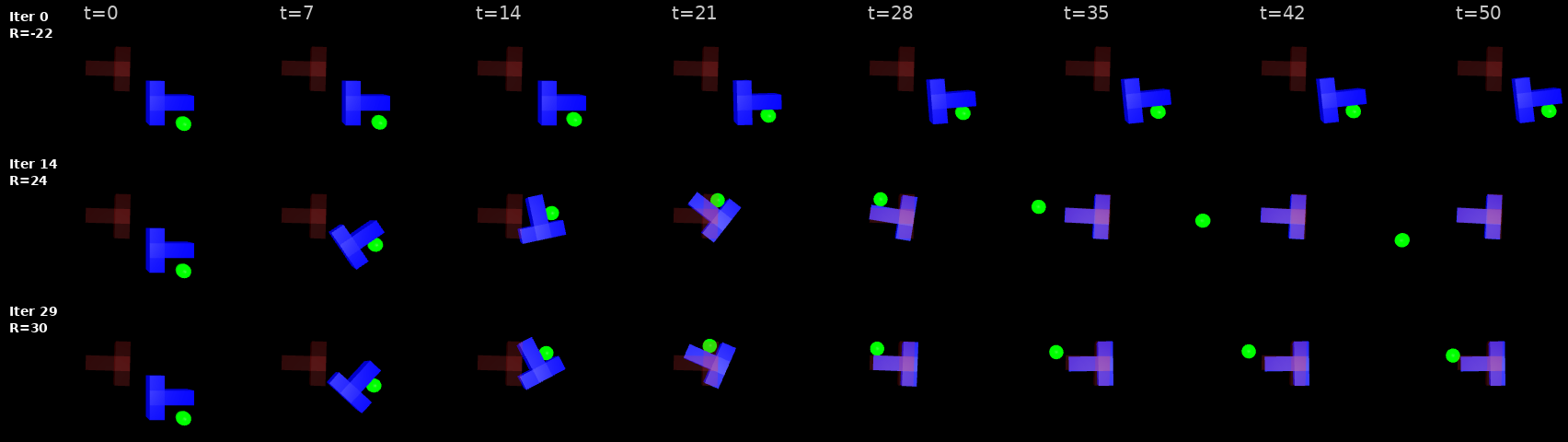}
    \caption{The PushT Task. Optimization time: $260.62 (s)$. }
    \end{subfigure}
    \begin{subfigure}{\linewidth}
    \centering
    \vspace{0.2 cm}
        \includegraphics[width=0.9\linewidth]{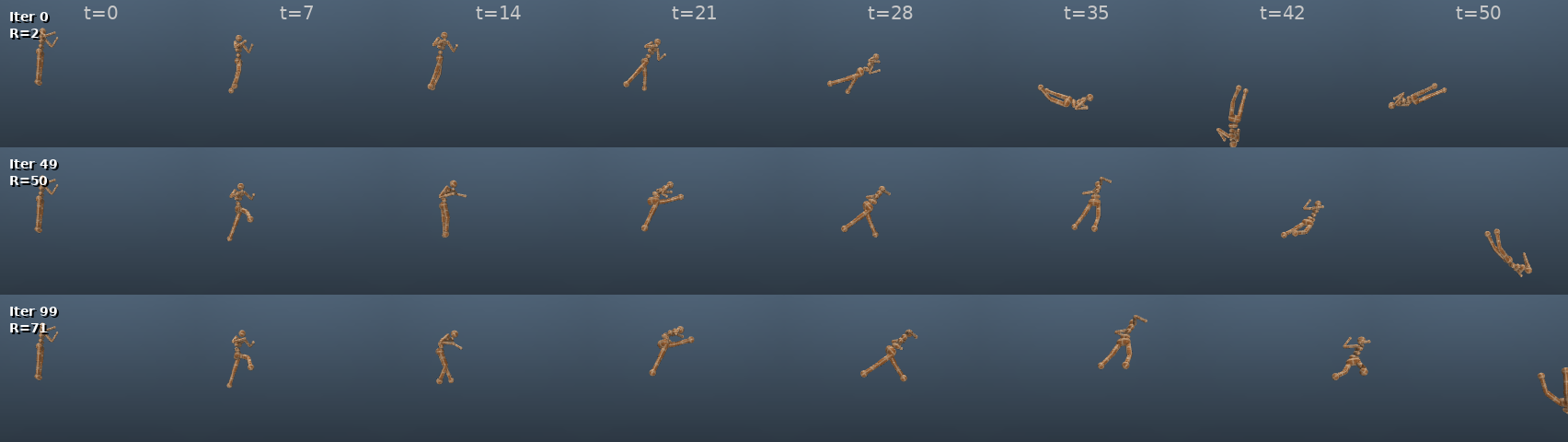}
    \caption{The Humanoid Running Task. Optimization time: $142.21 (s)$.}
    \end{subfigure}
    \begin{subfigure}{\linewidth}
    \centering
    \vspace{0.2 cm}
        \includegraphics[width=0.9\linewidth]{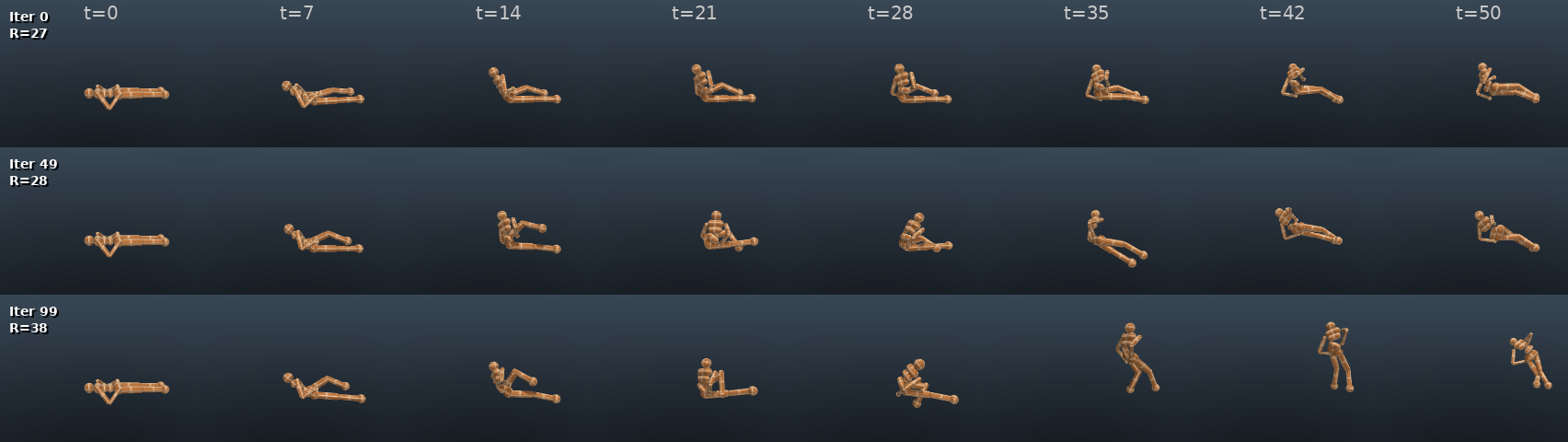}
    \caption{The Humanoid Standing Up Task. Optimization time: $168.11 (s)$. }
    \end{subfigure}
    \caption{The optimization results for contact-rich tasks.}
    \label{fig:humanoidstandup}
\end{figure*}

\begin{figure*}[h]
    \centering
    \begin{subfigure}{0.225\textwidth}
    \centering
    \includegraphics[width=\textwidth]{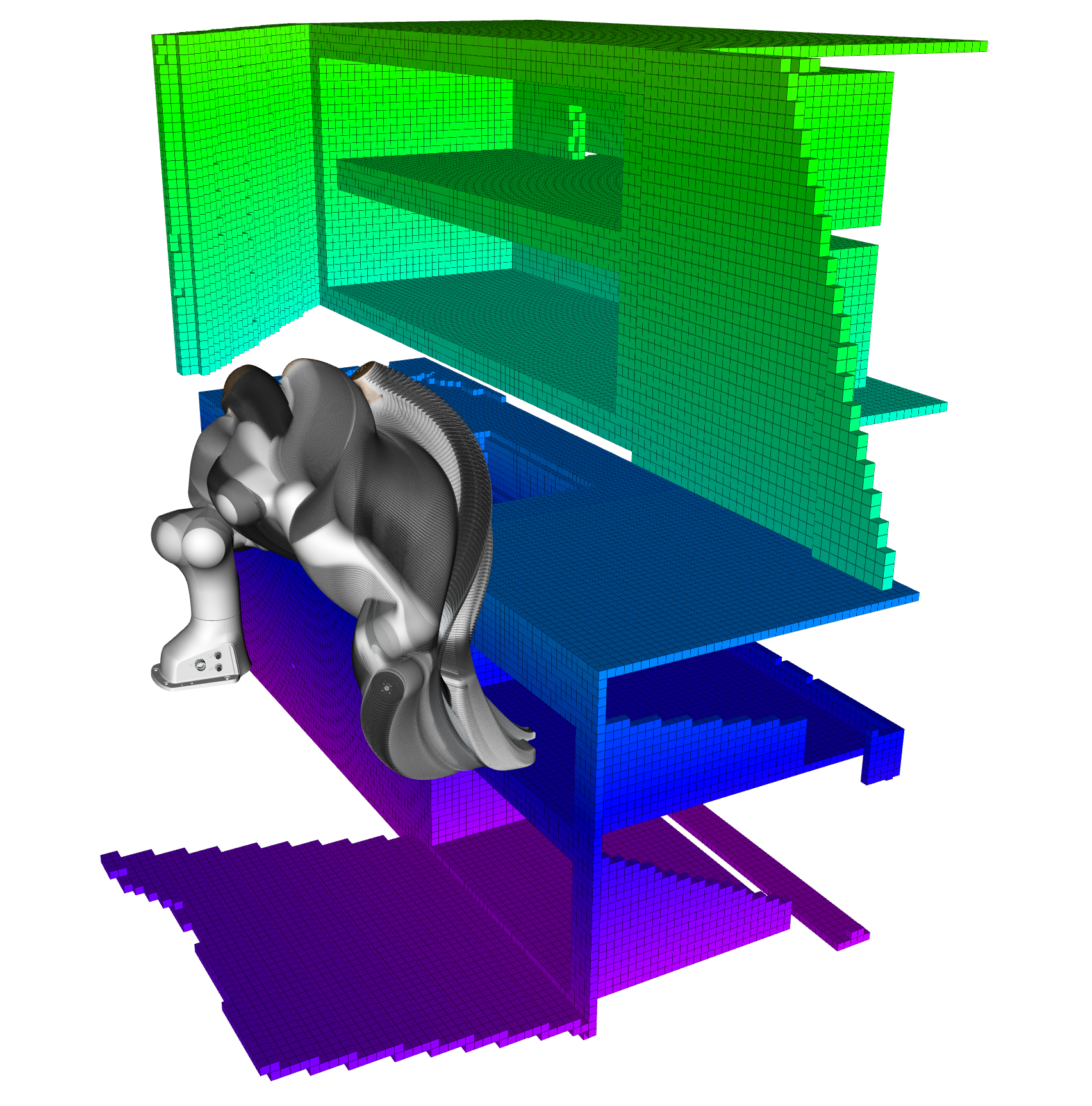}
    \caption{PISTO}
    \end{subfigure}
    \begin{subfigure}{0.225\textwidth}
    \centering
    \includegraphics[width=\textwidth]{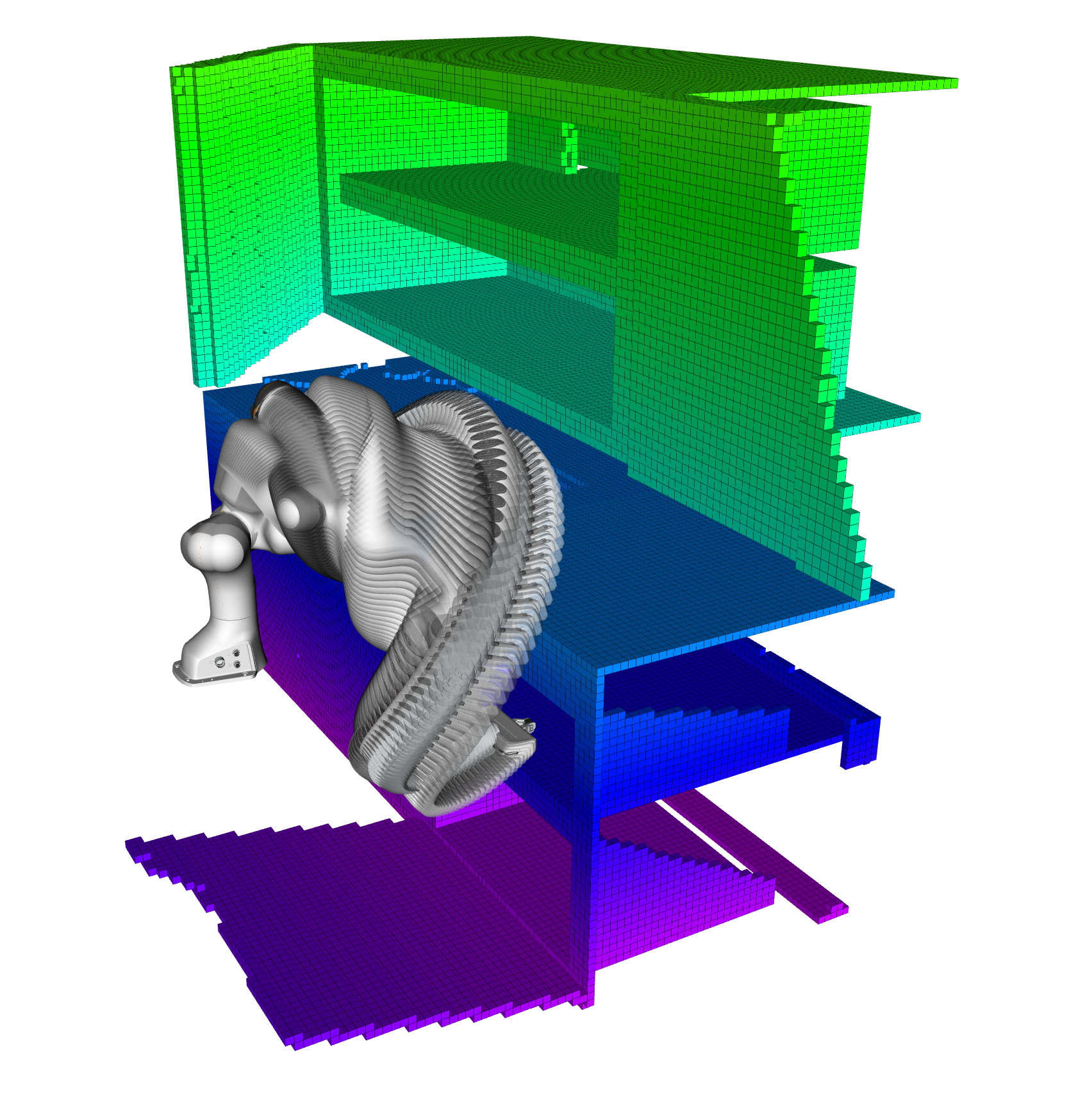}
    \caption{NGD}
    \end{subfigure}
    \begin{subfigure}{0.225\textwidth}
    \centering
    \includegraphics[width=\textwidth]{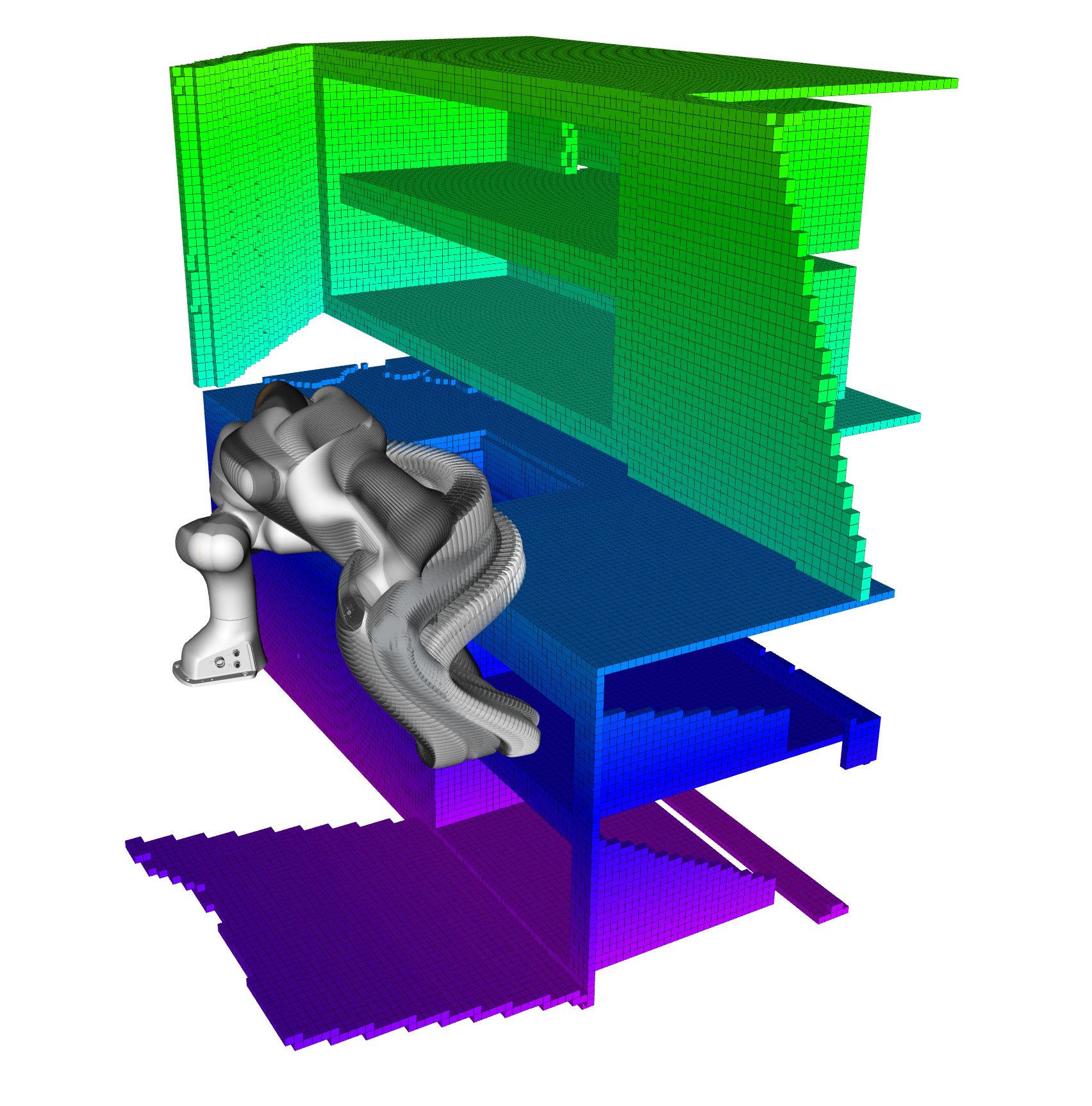}
    \caption{STOMP}
    \end{subfigure}
    \begin{subfigure}{0.225\textwidth}
    \centering
    \includegraphics[width=\textwidth]{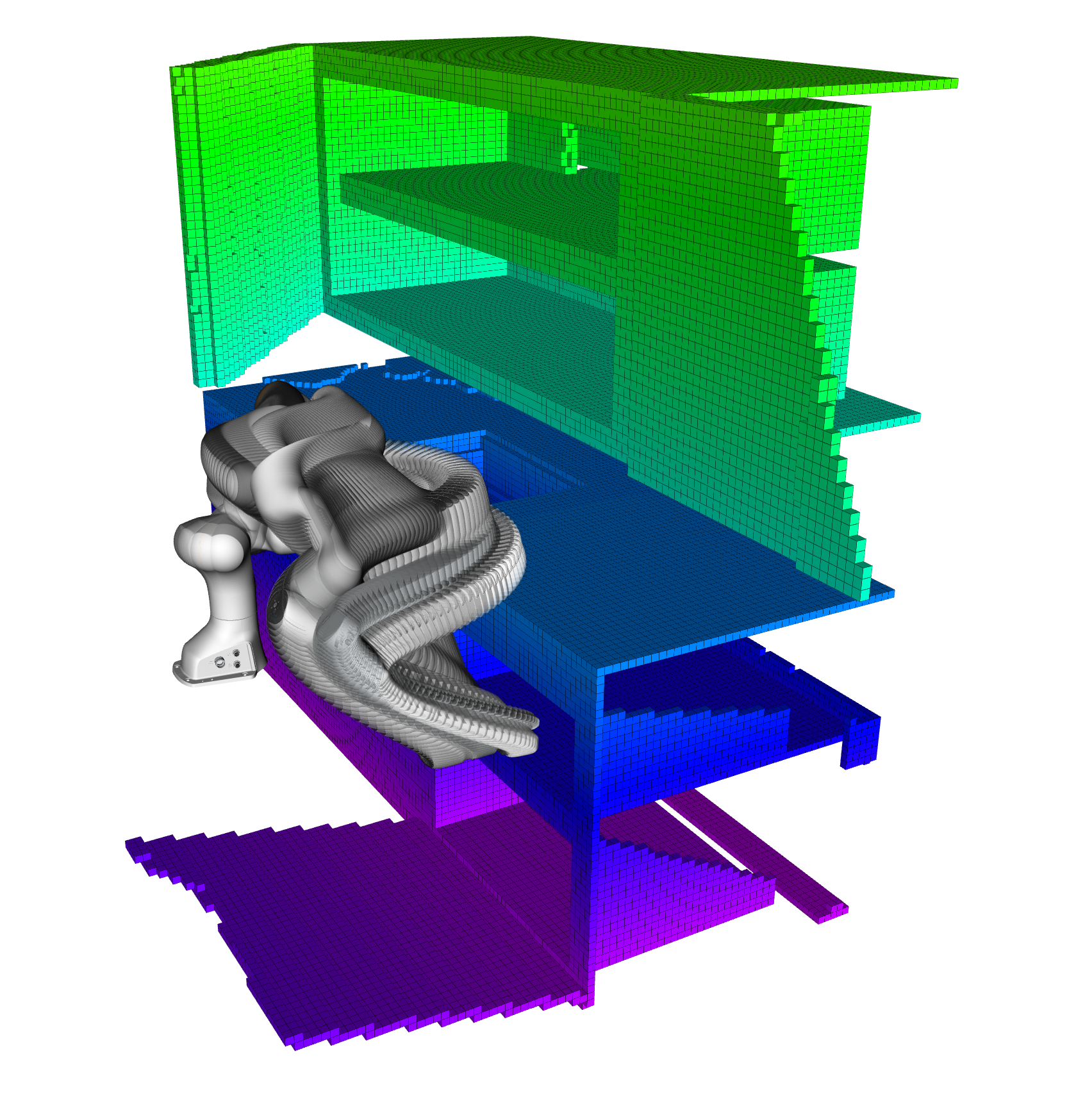}
    \caption{CHOMP}
    \end{subfigure}
    \caption{Results of different planners in the Kitchen scene in the database.}
\end{figure*}

\subsection{Policy Optimization for Contact-Rich Tasks}
\label{sec:rl_to}
We now extend PISTO to policy optimization in the action space, where the decision variable is the \emph{control sequence} $U = \{u_t\}_{t=0}^{T-1}$ and states evolve according to known dynamics models.

\paragraph{Formulation}
Given initial state $x_0 \in \mR^{d_x}$, consider the finite-horizon optimal control problem:
\begin{equation}
\label{eq:mpc_objective}
\min_{U} \; \sum_{t=0}^{T-1} c(x_t, u_t)
\end{equation}
subject to the dynamics and control constraints:
\begin{align}
x_{t+1} &= g(x_t, u_t), \quad t = 0, \ldots, T-1, \label{eq:dynamics} \\
u_t &\in \mathcal{U} \subseteq \mR^{d_u}, \quad t = 0, \ldots, T-1, \label{eq:control_bounds}
\end{align}
where $U = (u_0, \ldots, u_{T-1}) \in \mR^{T \times d_u}$, $c: \mR^{d_x} \times \mR^{d_u} \to \mR$ is the stage cost, $g: \mR^{d_x} \times \mR^{d_u} \to \mR^{d_x}$ is the dynamics model, and $\mathcal{U} = [\underline{u}, \bar{u}]$ is the admissible control set.

\paragraph{Composite Structure of the Objective}

The state trajectory is uniquely determined by the initial condition and control sequence via recursive application of \eqref{eq:dynamics}. We formalize this through the \emph{flow map} $\phi_t: \mR^{d_x} \times \mR^{t \times d_u} \to \mR^{d_x}$:
\begin{equation}
\label{eq:flow_map}
x_t = \phi_t(x_0, u_{0:t-1}) \triangleq \underbrace{g \circ g \circ \cdots \circ g}_{t \text{ times}}(x_0, u_{0:t-1}),
\end{equation}
with the convention $\phi_0(x_0) = x_0$. Substituting into \eqref{eq:mpc_objective} and defining the \emph{rollout cost}
\begin{equation}
\label{eq:rollout_cost}
S(U; x_0) \triangleq \sum_{t=0}^{T-1} c\bigl(\phi_t(x_0, u_{0:t-1}), u_t\bigr),
\end{equation}
the optimal control problem reduces to the form of \eqref{eq:obj_stomp}:
\begin{equation}
\label{eq:policy_opt_stomp}
\min_U \; \mE_{\tilde{U}} \left[ S(\tilde{U}; x_0) + \frac{1}{2} \tilde{U}^\top R \tilde{U} \right],
\end{equation}
where the quadratic term $\frac{1}{2} \tilde{U}^\top R \tilde{U}$ regularizes the control sequence for temporal smoothness. Algorithm~\ref{alg:pisto} applies directly in this setting, with each sample $\tilde U^{(i)} \sim \cN(U, \Sigma)$ evaluated by rolling out the dynamics $\phi_t$ to obtain $S(\tilde U^{(i)}; x_0)$. This formulation enables PISTO to optimize directly in control space without requiring differentiability of $g$ or $c$, and admits efficient parallelization of rollouts on modern GPU hardware.

\subsection{Covariance and Proximal Step Size Annealing}

To balance exploration and exploitation, we implement covariance scheduling with adaptive temperature scaling. We scale the smoothness matrix, $\bar{R} = \sigma_k \times R$, and use $\bar{R}$ matrix as the actual matrix that we sample from. The covariance scale $\sigma_k$ follows cosine annealing:
\begin{equation}
    \sigma_k = \sigma_{\text{final}} + \frac{1}{2}(\sigma_{\text{init}} - \sigma_{\text{final}})\left(1 + \cos\left(\frac{\pi k}{\Kmax}\right)\right),
\end{equation}
where $k$ is the current iteration and $\Kmax$ the maximum. We introduce an adaptive annealing scheme for the proximal step size parameter $\eta$ governing the exploration-exploitation trade-off in importance-weighted trajectory optimization. We also introduce an additional temperature parameter $\tau$ to scale the impact of the adaptive proximal step size. The importance weights are $w_m \propto \exp\left(-\frac{\gamma}{\tau} S(Y_k + \varepsilon_m) - (\varepsilon_m)^\top R Y_k\right)$. We anneal $\eta$ from small to large values via an exponential schedule $\eta(t) = \eta_{\text{initial}} \cdot (\eta_{\text{final}} / \eta_{\text{initial}})^{t/T}$, where small $\eta$ encourages exploration through nearly uniform weights while large $\eta$ exploits high-reward samples with peaked weights. The temperature $\tau$ amplifies only the energy term, leaving regularization unaffected, enabling independent control over reward sensitivity and trajectory smoothness
\footnotetext{For STOMP, we employed the official implementation \url{https://github.com/ros-industrial/stomp}}.
\begin{figure*}[h]
    \centering
    \begin{subfigure}{0.225\textwidth}
    \centering
    \includegraphics[width=\textwidth]{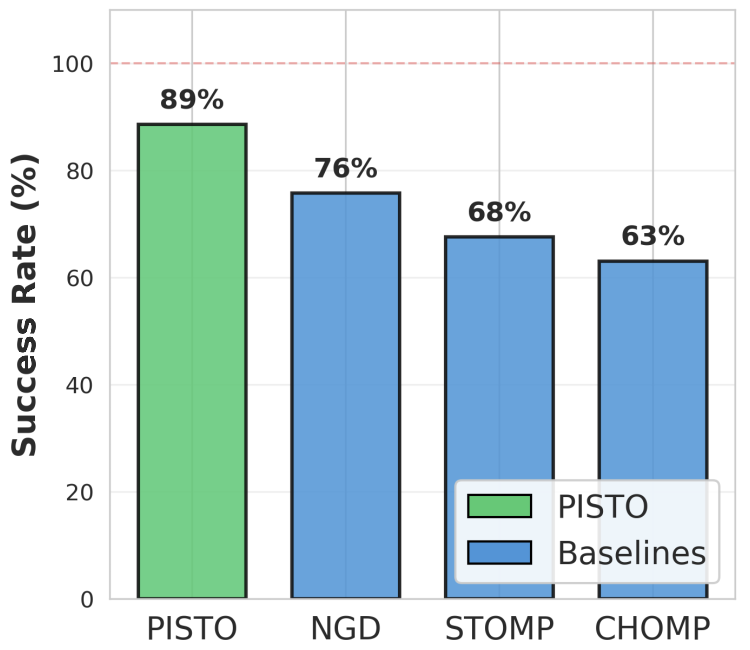}
    \caption{Success Rate}
    \end{subfigure}
    \begin{subfigure}{0.225\textwidth}
    \centering
    \includegraphics[width=\textwidth]{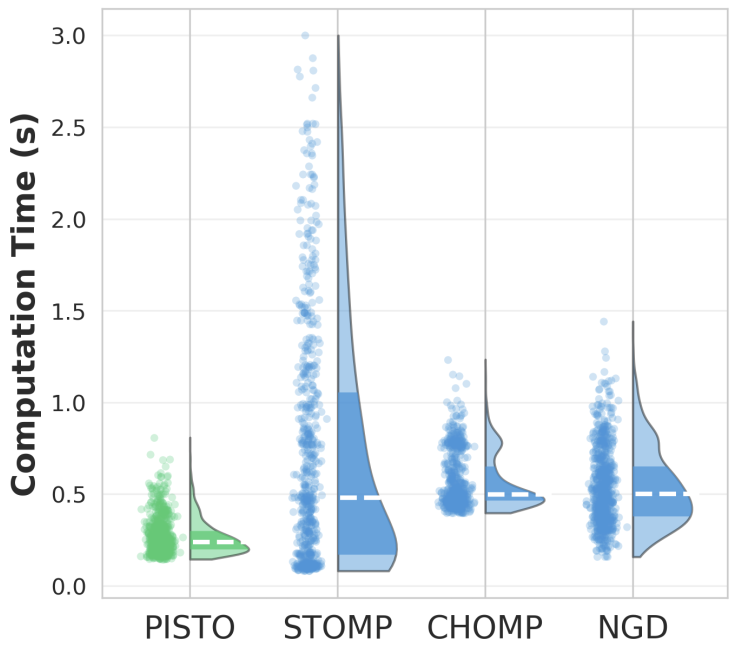}
    \caption{Planning Time}
    \end{subfigure}
    \begin{subfigure}{0.225\textwidth}
    \centering
    \includegraphics[width=\textwidth]{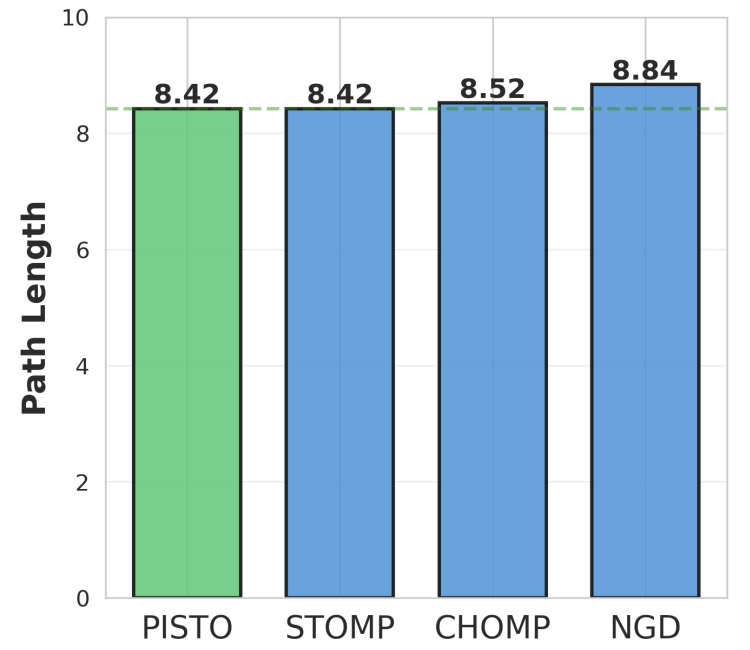}
    \caption{Path Length}
    \end{subfigure}
    \begin{subfigure}{0.225\textwidth}
    \centering
    \includegraphics[width=\textwidth]{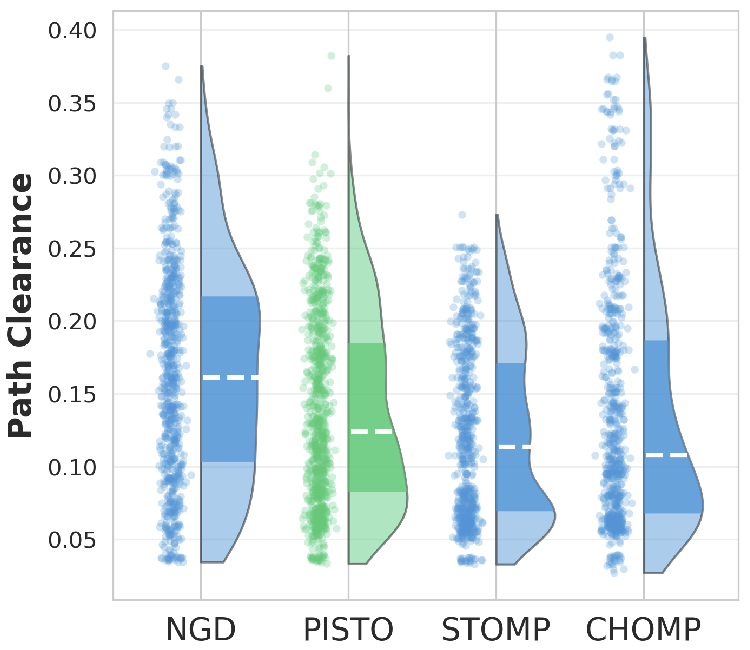}
    \caption{Path Clearance}
    \end{subfigure}
    \caption{Benchmark Performance Statistics}
    \label{fig:benchmarking_result}
\end{figure*}

\section{Experiments}
All experiments were run on a computer with an Intel Core i7-12800H CPU. The MuJoCo experiments were run on a computer with an NVIDIA RTX 4090 GPU. The code for this paper is implemented in C++ for the motion planning tasks, and in Python for the trajectory optimization tasks in MuJoCo. 

\subsection{Motion Planning for Robot Arms}
For collision avoidance, we define the state cost $V(Y_t)$ using two formulations. The \emph{signed-distance} cost approximates the robot as a union of spheres and evaluates
\begin{equation}
    V_{\rm coll}(Y_t) = \big\| h_{\delta}\left( \dsdf(F(Y_t)) \right) \big\|_{\Sigma_{\rm obs}}^{2},
\end{equation}
where $F(\cdot)$ is the forward kinematics, $\dsdf(\cdot)$ queries a precomputed signed distance field (SDF), and the hinge function $h_{\delta}(\cdot)$ penalizes penetrations within margin $\delta$. The non-differentiable \emph{indicator} cost, used in MoveIt benchmarking experiments, imposes a fixed penalty on detected collision:
\begin{equation}
    V_{\rm coll}(Y_t) = W_{\rm obs}\, \mathbf{1}_{\mathrm{CollisionDetected}(Y_t)}.
\end{equation}
The weights $\Sigma_{\rm obs}$ and $W_{\rm obs}$ are task-dependent hyper-parameters.

\subsection{Policy Optimization Tasks}
We tested the PISTO algorithm on various tasks defined in MuJoCo~\cite{todorov2012mujoco}. We used Bayesian optimization-based parameter sweeping to obtain the recommended parameters for each task. Figure~\ref{fig:humanoidstandup} illustrates the optimization process for a standing up task for a $17$-DOF humanoid robot. Table~\ref{tab:combined_benchmarking} records the achieved rewards for different tasks.
\textit{Planning and Runtime Results: }
Table \ref{tab:combined_benchmarking} records PISTO's performance in contact-rich trajectory optimization tasks, compared with CEM and MPPI methods. The results are averaged over $50$ independent runs. Our proposed PISTO algorithm significantly outperforms both baselines across all five contact-rich tasks, achieving approximately $1.5\times$ improvement on Walker2d and $2.8\times$ on HumanoidRun compared to the best baseline, while converting negative MPPI performance on PushT ($-0.17$) into a positive reward of $0.46$. PISTO also runs $1.4\times$--$3.2\times$ faster than baselines on most tasks, with particularly notable speedups on HumanoidRun ($36.5$s vs.\ $\sim$115s). The tight standard deviations indicate reliable convergence despite discontinuous contact dynamics. Notably, PISTO achieves successful results for the $17$-DOF humanoid running and standing-up tasks within minutes. These results validate PISTO's robustness for trajectory optimization in hybrid dynamical systems.

\begin{table*}[htbp]
\centering
\caption{Benchmarking Results: Motion Planning and Trajectory Optimization}
\label{tab:combined_benchmarking}
\begin{tabular}{llcccccccc}
\toprule
\textbf{Task} & \textbf{Method} & &\textbf{Success/Reward} & \textbf{Time (s)} & \textbf{Path Length} & \textbf{Path Clearance} \\
\midrule
\multicolumn{6}{l}{\textit{Motion Planning (7-DOF Manipulator)}} \\
\midrule
& CHOMP & & 63.05\% & 0.498 & 8.519 & 0.108 \\
& STOMP & & 67.59\% & \underline{0.481} & \underline{8.424} & 0.113 \\
& NGD & & \underline{75.76}\% & 0.499 & 8.838 & \textbf{0.161} \\
& PISTO & & \textbf{88.57}\% & \textbf{0.237} & \textbf{8.421} & \underline{0.124} \\
\midrule
& \multicolumn{3}{c}{\textbf{Reward} (Per step)} & \multicolumn{3}{c} {\textbf{Runtime}} \\
\cmidrule(lr){2-4} \cmidrule(lr){5-8} 
\textbf{Task} & PISTO & CEM & MPPI & PISTO & CEM & MPPI & \\
\midrule
\multicolumn{8}{l}{\textit{Trajectory Optimization for Contact-Rich Tasks}} \\
\midrule
 PushT             & $\textbf{0.4559} \pm 0.1291$ & $0.1004 \pm 0.1416$ & $-0.1715 \pm 0.1407$ & $134.6235s \pm 0.8061 s$ & $229.17s \pm 1.48s$ & $228.55s \pm 1.44s$ &\\
 Hopper             & $\textbf{1.2645} \pm 0.0201$ & $0.6931 \pm 0.1390$ &  $0.9195 \pm 0.1222$ & $49.13s ± 1.51s$ & $58.76s \pm 2.61s$  & $58.53s \pm 2.47s $  & \\
 Walker2d          & $\textbf{1.2622} \pm 0.0828$ & $0.8603 \pm 0.1107$ & $0.7505 \pm 0.1642$ & $65.99s \pm 1.50s$ & $75.13s \pm 2.57s$ & $75.14s \pm 2.72s$ &\\
 HumanoidRun       & $\textbf{1.3385} \pm 0.2921$ & $0.4106 \pm 0.2509$ & $0.485 \pm 0.3104$ & $36.48s \pm 1.64s$ &  $117.62s \pm 14.42s$ & $112.91s \pm 12.98s$ &\\
 HumanoidStandUp   & $\textbf{0.679} \pm 0.058$ & $0.4609 \pm 0.008$ & $0.529 \pm 0.0390$ & $65.29s \pm 2.02s$ & $55.18s \pm 1.44s$ & $54.24s \pm 2.94s$ & \\
\bottomrule
\end{tabular}
\end{table*}

\subsection{Benchmarking with MoveIt Motion Planning Algorithms}
% \paragraph{Benchmarking with General-purpose Nonlinear Optimization Solvers}
% To isolate the performance of our optimization approach, we benchmark the deterministic mean trajectory solver against standard nonlinear optimizers. For the collision cost $V_{\text{coll}}$, we adopt the smooth signed-distance formulation with hinge function $h_{\delta}$, ensuring differentiability for gradient-based optimization. All gradients are computed using CasADi's automatic differentiation \cite{andersson2019casadi}. We compare against IPOPT \cite{Wachter2006}, L-BFGS \cite{Liu1989}, and SQP \cite{boggs1995sequential}. The results showcase that PISTO is less sensitive to the initialization. Thanks to the smoothness matrix $R$ which enforces correlations between states, PISTO produces smoother trajectories compared with baselines. As a result, PISTO's results are more desirable than those of the vanilla general-purpose solvers.  

% \paragraph{Benchmarking with other Motion Planning Algorithms in MoveIt}
We further benchmark PISTO against several representative optimization-based baselines using MoveIt, including the official implementations of STOMP and CHOMP, as well as a natural gradient descent (NGD) method \cite{yu2023gvimp}.
Experiments are conducted on the Franka Emika Panda across seven manipulation environments from the MotionBenchMaker~\cite{chamzas2021motionbenchmaker} dataset, including \emph{Kitchen}, \emph{Bookshelf Tall}, \emph{Bookshelf Thin}, \emph{Table Pick}, \emph{Table Under Pick}, \emph{Box}, and \emph{Cage}, with over 300 planning tasks in total. 

For a fair comparison, all planners are initialized using the same joint-space straight-line interpolation between the start and goal configurations. We repeat each planning task 20 times with different random seeds. A run is considered successful if the resulting trajectory is collision-free and satisfies joint limits. Figure \ref{fig:benchmarking_result} summarizes the benchmarking results in terms of success rate, planning time, path length, and path clearance.

\section{Conclusion}
\label{sec:conclusion}
We presented the Proximal Inference for Stochastic Trajectory Optimization (PISTO), a principled algorithm for motion planning formulated as Gaussian variational inference. By revealing STOMP's implicit variational structure, we introduced a proximal formulation that regularizes updates via KL penalties, yielding closed-form moment-matching updates amenable to importance sampling. The resulting algorithm is simple, derivative-free, and parallelizable. Experiments demonstrated that PISTO achieves 89\% success on motion planning benchmarks—outperforming CHOMP, STOMP, and natural gradient baselines—while additional MuJoCo experiments validated its effectiveness for high-dimensional, contact-rich tasks. Future work includes incorporating gradients when available, jointly optimizing mean and covariance, and extending to receding-horizon MPC.

\appendix

\subsection{Proofs}
\label{sec:appendix_proofs}

\paragraph{Proof of Theorem \ref{thm:stomp_vi}}
Define the accumulated state cost $S(\Tilde{Y}) = \sum_{t=0}^T V(\Tilde{Y}_t)$. Since the covariance $\Sigma$ is fixed, the expectation of the quadratic deviation term simplifies via the trace identity $\mathbb{E}[(\Tilde{Y} - Y)^\top R (\Tilde{Y} - Y)] = \operatorname{tr}(R\Sigma)$. Thus,
\begin{equation*}
\mathcal{J}_1 = \mathbb{E}_{\Tilde{Y}}[S(\Tilde{Y})] + \frac{1}{2}\operatorname{tr}(R\Sigma) + \frac{1}{2}Y^\top R Y.
\end{equation*}

We recognize that $\mathbb{E}_{\Tilde{Y}}[S(\Tilde{Y})] = -\mathbb{E}_{\Tilde{Y}}[\log e^{-S(\Tilde{Y})}]$. Combining this with the KL divergence between Gaussians,
\begin{align*}
&D_{\mathrm{KL}}\left( \mathcal{N}(Y, \Sigma) \| \mathcal{N}(0, R^{-1}) \right) \nonumber
\\
=& \frac{1}{2}\left( \operatorname{tr}(R\Sigma) + Y^\top R Y - \log\det(\Sigma R) - (T+1) \right)
\end{align*}
where the last two terms are constants, we obtain
\begin{equation}
\mathcal{J}_1 = D_{\mathrm{KL}}\left( \mathcal{N}(Y, \Sigma) \| \mY^\star \right) + \text{const},
\end{equation}
by introducing the un-normalized $\mY^\star \propto e^{-S(\Tilde{Y})} \mathcal{N}(0, R^{-1})$. Here, the constant is independent of $Y$. Since only the mean $Y$ is optimized, minimizing $\mathcal{J}_1$ is equivalent to minimizing the KL divergence to the target posterior $\mY^\star$.

\paragraph{Proof of Theorem \ref{thm:prox_vi}}
Let $q_Y$ and $q_{Y_k}$ denote the density functions of $\cN(Y, \Sigma)$ and $\cN(Y_k, \Sigma)$, respectively. The proximal objective can be written as
\begin{align*}
\mathcal{J}_k &= D_{\mathrm{KL}} \left( \cN(Y,\Sigma) \parallel \mY^\star \right) + \frac{1}{\eta} D_{\mathrm{KL}}\left(\cN(Y,\Sigma) \parallel \cN(Y_k,\Sigma)\right) \nonumber \\
&= \mE_{\cN(Y,\Sigma)}\left[ \left( 1 + \frac{1}{\eta} \right) \log q_Y - \log \mY^\star - \frac{1}{\eta} \log q_{Y_k} \right].
\end{align*}
Factoring out the coefficient $\frac{\eta+1}{\eta}$ yields
\begin{align*}
\mathcal{J}_k &= \frac{\eta+1}{\eta} \, \mE_{\cN(Y,\Sigma)} \left[ \log q_Y - \log\left( (\mY^\star)^{\frac{\eta}{\eta+1}} (q_{Y_k})^{\frac{1}{\eta+1}} \right) \right] \nonumber \\
&\propto D_{\mathrm{KL}} \left( \cN(Y, \Sigma) \parallel \mY^\star_k \right),
\end{align*}
where we identify the surrogate distribution as
\begin{equation*}
\mY^\star_k \propto (\mY^\star)^{\frac{\eta}{\eta+1}} (q_{Y_k})^{\frac{1}{\eta+1}}.
\end{equation*}
Substituting the explicit form of $\mY^\star$ completes the proof.

\subsection{Implementation Details}
% \paragraph{Smoothness Matrix Normalization}
% To ensure the independence of matrix $R$ from the total planning time, we discretize the \textit{normalized} curvature cost $J_{\rm smooth} = \int_0^1 (d^2y/ds^2)^2 \, ds$ using central finite differences with spacing $\Delta s = 1/(N-1)$. The second-difference operator $\tilde{D} \in \mathbb{R}^{(N-2) \times N}$ applies the stencil $[1, -2, 1]$ at interior points, yielding $R = \frac{1}{\Delta s^3} \tilde{D}^\top \tilde{D} + \epsilon_{\rm reg} I$, where the scaling $(N-1)^3$ ensures \textit{discretization-invariant} cost and $\epsilon_{\rm reg} I$ provides numerical stability. The resulting symmetric pentadiagonal structure enables efficient sampling from the Gaussian prior $\mathcal{N}(0, R^{-1})$ via Cholesky factorization, biasing trajectories toward smooth, low-curvature solutions.

\paragraph{Elite Sample Selection}

To improve convergence and reduce variance, we incorporate elite-set selection within the importance sampling step. Given $M$ rollouts, we define an elite subset $\mathcal{M}_e \subset \{1, \dots, M\}$ containing samples in the lowest $\Kelite$-th percentile of cost. Importance weights for $m \notin \mathcal{M}_e$ are set to zero, while for $m \in \mathcal{M}_e$:
\begin{equation}
    w_m = \frac{\exp(-\gamma J(Y_m) / \tau)}{\sum_{j \in \mathcal{M}_e} \exp(-\gamma J(Y_j) / \tau)},
\end{equation}
where $J(\cdot)$ is the trajectory cost and $\tau$ is an adaptive temperature. This selective pressure ensures updates are driven by successful explorations, enabling efficient navigation of non-convex cost landscapes.

\paragraph{Momentum-Accelerated Exponential Moving Average}
To stabilize convergence and accelerate optimization, we apply a momentum-based update scheme. Let $\hat{Y}_{k+1}$ denote the candidate trajectory computed as the weighted expectation over the elite set $\hat{Y}_{k+1} = \sum_{m \in \mathcal{M}_e} w_m (Y_k + \varepsilon_m).$ We compute the update direction $\Delta_k = \hat{Y}_{k+1} - Y_k$ and maintain a momentum buffer $v_k$ via exponential moving average:
\begin{equation*}
    v_{k+1} = \beta v_k + (1 - \beta) \Delta_k,
\end{equation*}
where $\beta \in [0, 1)$ is the momentum decay coefficient. The trajectory is then updated as: $Y_{k+1} = Y_k + \lambda v_{k+1},$ where $\lambda \in (0, 1]$ is the step size. This two-stage temporal regularization dampens oscillations through directional smoothing while providing momentum for navigating non-smooth cost regions. We also allow the Adam-type \cite{kingma2014adam} gradient update rule as an option.

\bibliographystyle{plainnat}
\bibliography{references}

\end{document}